\documentclass[sigconf]{acmart}
\pdfoutput=1
\usepackage{algorithm}
\usepackage{algorithmic}
\usepackage{smile}
\usepackage{bbm}

\newcommand{\ifcomments}{\iftrue}

\AtBeginDocument{%
  \providecommand\BibTeX{{%
    \normalfont B\kern-0.5em{\scshape i\kern-0.25em b}\kern-0.8em\TeX}}}

\copyrightyear{2020} 
\acmYear{2020} 
\setcopyright{rightsretained} 
\acmConference[KDD '20]{Proceedings of the 26th ACM SIGKDD Conference on Knowledge Discovery and Data Mining}{August 23--27, 2020}{Virtual Event, CA, USA}
\acmBooktitle{Proceedings of the 26th ACM SIGKDD Conference on Knowledge Discovery and Data Mining (KDD '20), August 23--27, 2020, Virtual Event, CA, USA}
\acmDOI{10.1145/3394486.3403225}
\acmISBN{978-1-4503-7998-4/20/08}
\begin{document}
\fancyhead{}

\title{RayS: A Ray Searching Method for Hard-label Adversarial Attack}

\author{Jinghui Chen}
% \orcid{1234-5678}
\affiliation{%
  \institution{University of California, Los Angeles}
%   \city{XXX}
%   \state{XXX}
}
\email{jhchen@cs.ucla.edu}

\author{Quanquan Gu}
\affiliation{%
  \institution{University of California, Los Angeles}
%   \city{XXX}
%   \state{XXX}
}
\email{qgu@cs.ucla.edu}

\renewcommand{\shortauthors}{Anonymous Authors}

\begin{abstract}
Deep neural networks are vulnerable to adversarial attacks. Among different attack settings, the most challenging yet the most practical one is the hard-label setting where the attacker only has access to the hard-label output (prediction label) of the target model. Previous attempts are neither effective enough in terms of attack success rate nor efficient enough in terms of query complexity under the widely used $L_\infty$ norm threat model. 
In this paper, we present the Ray Searching attack (RayS), which greatly improves the hard-label attack effectiveness as well as efficiency. Unlike previous works, we reformulate the continuous problem of finding the closest decision boundary into a discrete problem that does not require any zeroth-order gradient estimation.
In the meantime, all unnecessary searches are eliminated via a fast check step. This significantly reduces the number of queries needed for our hard-label attack.
Moreover, interestingly, we found that the proposed RayS attack can also be used as a sanity check for possible ``falsely robust'' models. On several recently proposed defenses that claim to achieve the state-of-the-art robust accuracy, our attack method demonstrates that the current white-box/black-box attacks could still give a false sense of security and the robust accuracy drop between the most popular PGD attack and RayS attack could be as large as $28\%$. We believe that our proposed RayS attack could help identify falsely robust models that beat most white-box/black-box attacks.

\end{abstract}


 

\maketitle

\vspace{-0.2cm}
\section{Introduction}
Deep neural networks (DNNs) have achieved remarkable success on many machine learning tasks such as computer vision \citep{he2016deep, sutskever2012imagenet}, and speech recognition \citep{hinton2012deep} in the last decade. Despite the great success, recent studies have shown that DNNs are vulnerable to adversarial examples, i.e., even imperceptible (specially designed not random) perturbations could cause the state-of-the-art classifiers to make wrong predictions \citep{szegedy2013intriguing,goodfellow6572explaining}.  
This intriguing phenomenon has soon led to an arms race between adversarial attacks \citep{carlini2017towards,athalye2018obfuscated,chen2018frank}
that are trying to break the DNN models with such small perturbations and adversarial defenses methods \citep{papernot2016distillation, madry2017towards, wang2019convergence, zhang2019theoretically,Wang2020Improving} that tries to defend against existing attacks.
During this arm race, many heuristic defenses \citep{papernot2016distillation,guo2017countering,xie2017mitigating,song2017pixeldefend,ma2018characterizing,samangouei2018defense,dhillon2018stochastic} are later proved to be not effective under harder attacks. 
One exception is adversarial training \citep{goodfellow6572explaining,madry2017towards}, which was demonstrated as an effective defense approach.
 
A large body of adversarial attacks has been proposed during this arm race. According to the different amounts of information the attacker could access, adversarial attacks can be generally divided into three categories: white-box attacks, black-box attacks, and hard-label attacks.
White-box attacks \citep{madry2017towards, carlini2017towards} refer to the case where the attacker has access to all information regarding the target model, including the model weights, structures, parameters, and possible defense mechanisms. Since white-box attackers could access all model details, it can efficiently perform back-propagation on the target model and compute gradients.  
In black-box attacks, the attacker only has access to the queried soft label output (logits or probability distribution of different classes) of the target model, and the other parts are treated as a black-box. The black-box setting is much more practical compared with the white-box case, however, in such a setting, the attacker cannot perform back-propagation and direct gradient computation. Therefore, many turn to transfer the gradient from a known model \citep{papernot2016transferability} or estimate the true gradient via zeroth-order optimization methods \citep{ilyas2018black, ilyas2018prior, chen2018frank, AlDujaili2020Sign}.

Hard-label attacks, also known as decision-based attacks, on the other hand, only allow the attacker to query the target model and get hard-label output (prediction label). Obviously, the hard-label setting is the most challenging one, yet it is also the most practical one, as in reality, there is little chance that the attacker could know all the information about the target model in advance or get the probability prediction of all classes. 
The hard-label-only access also means that the attacker cannot tell the subtle changes in the target model's output when feeding a slightly perturbed input sample (assuming this slight perturbation will not change the model prediction).
Therefore, the attacker can only find informative clues around the decision boundary of the target model where tiny perturbations could cause the model to have different prediction labels. 
Previous works \cite{brendel2018decisionbased,cheng2018queryefficient,cheng2020signopt,chen2019hopskipjumpattack} mostly follow this idea to tackle the hard-label adversarial attack problem. 
However, \cite{brendel2018decisionbased,cheng2018queryefficient,cheng2020signopt,chen2019hopskipjumpattack}  are all originally proposed for $L_2$ norm threat model while $L_\infty$ norm threat models \citep{madry2017towards,zhang2019theoretically,kim2020sensible,zhang2019defense,zhang2020adversarial} are currently the most popular and widely used. Even though \cite{cheng2018queryefficient,cheng2020signopt,chen2019hopskipjumpattack} provide extensions to $L_\infty$ norm case, none of them has been optimized for the $L_\infty$ norm case and consequently, their attack performance falls largely behind traditional $L_\infty$ norm based white-box and black-box attacks, making them inapplicable in real world scenarios.
This leads to a natural question that,

\vspace{0.2cm}
\textit{Can we design a hard-label attack that could greatly improve upon previous hard-label attacks and provide practical attacks for the most widely used $L_\infty$ norm threat model?}
\vspace{0.2cm}

In this paper, we answer this question affirmatively.
We summarize our main contributions as follows 
\begin{itemize} 
    \item We propose the Ray Searching attack, which only relies on the hard-label output of the target model. We show that the proposed hard-label attack is much more effective and efficient than previous hard-label attacks in the $L_\infty$ norm threat model.
    
    \item Unlike previous works, most of which solve the hard-label attack problem via zeroth-order optimization methods, we reformulate the continuous optimization problem of finding the closest decision boundary into a discrete one and directly search for the closest decision boundary along a discrete set of ray directions. 
    A fast check step is also utilized to skip unnecessary searches.
    This significantly saves the number of queries needed for the hard-label attack. Our proposed attack is also free of hyperparameter tuning such as step size or finite difference constant, making itself very stable and easy to apply.
    
    \item Moreover, our proposed RayS attack can also be used as a strong attack to detect possible ``falsely robust'' models. By evaluating several recently proposed defenses that claim to achieve the state-of-the-art robust accuracy with RayS attack, we show that the current white-box/black-box attacks can be deceived and give a false sense of security.  Specifically, the RayS attack significantly decrease the robust accuracy of the most popular PGD attack on several robust models and the difference could be as large as $28\%$. We believe that our proposed RayS attack could help identify falsely robust models that deceive current white-box/black-box attacks.

\end{itemize}
 \vspace{-0.1cm}
The remainder of this paper is organized as follows: in Section \ref{sec:related}, we briefly review existing literature on adversarial attacks. We present our proposed Ray Searching attack (RayS) in Section \ref{sec:methods}. In Section \ref{sec:exp}, we show the proposed RayS attack is more efficient than other hard-label attacks and can be used as a sanity check for detecting falsely robust models by evaluating several recently proposed defenses. Finally, we conclude this paper and provide discussions in Section \ref{sec:conclusion}. 
 
% \smallskip
 
\textbf{Notation.} For a $d$-dimensional vector $\xb = [x_1,...,x_d]^{\top}$, we use $\|\xb\|_{0}=\sum_i \mathbbm{1}\{x_i \neq 0\}$ to denote its $\ell_0$-norm, use $\|\xb\|_{2} = (\sum_{i=1}^{d}|x_{i}|^{2})^{1/2}$ to denote its $\ell_2$-norm and use $\|\xb\|_{\infty}=\max_i|x_i|$ to denote its $\ell_\infty$-norm, where $\mathbbm{1}(\cdot)$ denotes the indicator function.

\vspace{-0.2cm}
\section{Related Work}\label{sec:related}
There is a large body of works on evaluating model robustness and generating adversarial examples. In this section, we review the most relevant works with ours. 

\textbf{White-box attacks:}
\citet{szegedy2013intriguing} first brought up the concept of adversarial examples and adopt the L-BFGS algorithm for attacks. \citet{goodfellow6572explaining} proposed the Fast Gradient Sign Method (FGSM) method via linearizing the network loss function. \citet{kurakin2016adversarial} proposed to iteratively perform FGSM and conduct projection afterward, which is equivalent to Projected Gradient Descent (PGD) \citep{madry2017towards}. \citet{papernot2016limitations} proposed JSMA method based on the Jacobian saliency map and \citet{moosavi2016deepfool} proposed DeepFool attack by projecting the data to the closest separating hyper-plane.
\citet{carlini2017towards} introduced the CW attack with a margin-based loss function and show that defensive distillation \citep{papernot2016distillation} is not truly robust. \citet{chen2018frank} proposed a projection-free attack based on the Frank-Wolfe method with momentum. \citet{athalye2018obfuscated} identified the effect of obfuscated gradients and proposed the BPDA attack for breaking those obfuscated gradient defenses.

\textbf{Black-box attacks:}
Other than the aforementioned white-box attack algorithms, there also exists a large body of literature \citep{hu2017generating,papernot2016transferability,papernot2017practical,chen2017zoo,ilyas2018black,ilyas2018prior,li2019nattack,chen2018frank} focusing on the black-box attack case where the information is limited to the logits output of the model rather than every detail of the model.
Transfer-based black-box attacks \citep{hu2017generating,papernot2016transferability,papernot2017practical} try to transfer the gradient from a known model to the black-box target model and then apply the same technique as in the white-box case. However, their attack effectiveness is often not quite satisfactory. Optimization-based black-box attacks aim to estimate the true gradient via zeroth-order optimization methods. \citet{chen2017zoo} proposed to estimate the gradient via finite-difference on each dimension. \citet{ilyas2018black} proposed to improve the query efficiency of \cite{chen2017zoo} via Natural Evolutionary Strategies. \citet{ilyas2018prior} further improved upon \citet{ilyas2018black} by exploiting gradient priors.
\citet{uesato2018adversarial} proposed to use the SPSA method to build a gradient-free attack that can break vanishing gradient defenses.
\citet{AlDujaili2020Sign} proposed to directly estimate the sign of the gradient instead of the true gradient itself.
\citet{moon2019parsimonious} reformulated the continuous optimization problem into a discrete one and proposed a combinatorial search based algorithm to make the attack more efficient. \citet{andriushchenko2019square} proposed a randomized search scheme to iteratively patch small squares onto the test example.
% \cnote{Another steam of research focused game-theoretical adversarial learning...}

\textbf{Hard-label attacks:}  
\citet{brendel2018decisionbased} first studied the hard-label attack problem and proposed to solve it via random walks near the decision boundary.
\citet{ilyas2018black} demonstrated a way to transform the hard-label attack problem into a soft label attack problem.
\citet{cheng2018queryefficient} turned the adversarial optimization problem into the problem of finding the optimal direction that leads to the shortest $L_2$ distance to decision boundary and optimized the new problem via zeroth-order optimization methods. \citet{cheng2020signopt} further improved the query complexity of \citep{cheng2018queryefficient} by estimating the sign of gradient instead of the true gradient. \citet{chen2019hopskipjumpattack} also applied zeroth-order sign oracle to improve \citep{brendel2018decisionbased} by searching the step size and keeping the iterates along the decision boundary.

\vspace{-0.2cm}
\section{The Proposed Method}\label{sec:methods}
In this section, we introduce our proposed \textit{Ray Searching attack} (RayS). Before we go into details about our proposed method, we first take an overview of the previous adversarial attack problem formulations.

\subsection{Overview of Previous Problem Formulations}
We denote the DNN model by $f$ and the test data example as $\{\xb, y\}$. The goal of adversarial attack is to solve the following optimization problem
\begin{align}\label{eq:adv_attack}
    \min_{\xb'} \mathbbm{1}\{f(\xb') = y\} \ \text{ s.t., } \ \|\xb' - \xb\|_\infty \leq \epsilon,
\end{align}
where $\epsilon$ denotes the maximum allowed perturbation strength.
The indicator function $\mathbbm{1}\{f(\xb') = y\}$ is hard to optimize, therefore, \cite{madry2017towards,zhang2019theoretically,chen2018frank,ilyas2018black,ilyas2018prior,AlDujaili2020Sign} turn to relax \eqref{eq:adv_attack} into
\begin{align}\label{eq:adv_attack_ce}
    \max_{\xb'} \ell(f(\xb'), y) \ \text{ s.t., } \ \|\xb' - \xb\|_\infty \leq \epsilon,
\end{align}
where $\ell$ denotes the surrogate loss function such as CrossEntropy loss. 
On the other hand, traditional hard-label attacks \citep{cheng2018queryefficient,cheng2020signopt} re-formulate \eqref{eq:adv_attack} as
\begin{align}\label{eq:adv_hard_label}
    \min_{\db} g(\db) \ \text{ where } \ g(\db) = \argmin_r \mathbbm{1}\{f(\xb + r \db / \|\db\|_2) \neq y\}.
\end{align}
Here $g(\db)$ represents the decision boundary radius from original example $\xb$ along ray direction $\db$ and the goal is to find the minimum decision boundary radius regarding the original example $\xb$. Let $(\hat{r}, \hat{\db})$ denotes the minimum decision boundary radius and the corresponding ray direction. If the minimum decision boundary radius satisfies $\|\hat{r}   \hat{\db} / \|\hat{\db}\|_2\|_\infty \leq \epsilon$, it will be counted as a successful attack.

While prior works \citep{cheng2018queryefficient,cheng2020signopt} try to solve problem \eqref{eq:adv_hard_label} in a continuous fashion by estimating the gradient of $g(\db)$ via zeroth-order optimization methods, the hard-label-only access restriction imposes great challenges in solving \eqref{eq:adv_hard_label}. Specifically, estimating the the decision boundary radius $g(\db)$ typically takes a binary search procedure and estimating an informative gradient of $g(\db)$ via finite difference requires multiple rounds of $g(\db)$ computation. Furthermore, due to the large variance in zeroth-order gradient estimating procedure, optimizing \eqref{eq:adv_hard_label} typically takes a large number of gradient steps. These together, make solving \eqref{eq:adv_hard_label} much less efficient and effective than black-box attacks, not to mention white-box attacks.

Given all the problems mentioned above, we turn to directly search for the closest decision boundary without estimating any gradients. 

\subsection{Ray Search Directions}
With a finite number of queries, it is impossible to search through the whole continuous ray direction space. As a consequence, we need to restrict the search space to a discrete set of ray directions to make direct searches possible. Note that applying FGSM to \eqref{eq:adv_attack_ce} leads to an optimal solution at the vertex of the $L_\infty$ norm ball \citep{moon2019parsimonious,chen2018frank}, suggesting that those vertices might provide possible solutions to \eqref{eq:adv_attack_ce}. Empirical findings in \citep{moon2019parsimonious} also suggest that the solution to \eqref{eq:adv_attack_ce} obtained from the PGD attack is mostly found on the vertices of $L_\infty$ norm ball. Inspired by this, \citet{moon2019parsimonious} restrict the feasible solution set as the vertex of the $L_\infty$ norm ball. Following this idea, since our goal is to obtain the decision boundary radius, we consider the ray directions that point to the $L_\infty$ norm ball  vertices, i.e., $\db \in \{-1,1\}^d$ where $d$ denotes the dimension of original data example $\xb$\footnote{Without loss of generality, here we view $\xb$ simply as a $d$-dimensional vector.}. 
Therefore, instead of solving \eqref{eq:adv_hard_label}, we turn to solve a discrete problem 
 \begin{align}\label{eq:adv_hard_label_discrete}
    \min_{\db \in \{-1, 1\}^d} g(\db) \ \text{ where } \ g(\db) = \argmin_r \mathbbm{1}\{f(\xb + r \db / \|\db\|_2) \neq y\}.
\end{align}
In problem \eqref{eq:adv_hard_label_discrete}, we reduce the search space from $\RR^d$ to $\{-1, 1\}^d$, which contains $2^d$ possible search directions.

Now we begin to introduce our proposed Ray Searching attack.
We first present the naive version of the Ray Searching attack, which is summarized in Algorithm \ref{alg:rays_naive}.
Specifically, given a model $f$ and a test example $\{\xb, y\}$, we first initialize the best search direction as an all-one vector and set the initial best radius as infinity. Then we iteratively change the sign of each dimension of the current best ray direction and test whether this modified ray direction leads to a better decision boundary radius by Algorithm \ref{alg:bin_search} (will be described later). If it does, we update the best search direction and the best radius, otherwise, they remain unchanged. Algorithm \ref{alg:rays_naive} is a greedy search algorithm that finds the local optima of the decision boundary radius, where the local optima of the decision boundary radius are defined as follows.

\begin{definition}[Local Optima of Decision Boundary Radius]
A ray direction $\db \in \{-1,1\}^d$ is the local optima of the decision boundary radius regarding \eqref{eq:adv_hard_label_discrete}, if for all $\db' \in \{-1,1\}^d$ satisfy $\|\db' - \db\|_0 \leq 1$, we have $g(\db) \leq g(\db')$.
\end{definition}

\begin{theorem}
Given enough query budgets, let $(\hat{r}, \hat{\db})$ be the output of Algorithm \ref{alg:rays_naive}, then $\hat{\db}$ is the local optima of decision boundary radius problem \eqref{eq:adv_hard_label_discrete}.
\end{theorem}
\begin{proof}
We prove this by contradiction. Suppose $\hat{\db}$ is not the local optima, there must exist some $\db'$ satisfying $\|\db' - \hat{\db}\|_0 \leq 1$, i.e., $\db'$ differs from $\hat{\db}$ by at most $1$ dimension, that $g(\hat{\db}) > g(\db')$. This means Algorithm \ref{alg:rays_naive} can still find better solution than $g(\hat{\db})$ by going through all dimensions and thus $\hat{\db}$ will not be the output of Algorithm \ref{alg:rays_naive}. This leads to a contradiction.
\end{proof}

Next we introduce Algorithm \ref{alg:bin_search}, which performs decision boundary radius search. The main body of Algorithm \ref{alg:bin_search} (from Line \ref{line:bin_start} to Line \ref{line:bin_end}) is a binary search algorithm to locate the decision boundary radius with high precision. The steps before Line \ref{line:bin_start}, on the other hand, focus on deciding the search range and whether we need to search it (this is the key to achieve efficient attacks).
Specifically, we first normalize the search direction by its $L_2$ norm. And then in Line \ref{line:fast_check}, we do a fast check at $x + r_{\text{best}} \cdot \db_n$\footnote{For applications such as image classification, there is an additional clipping to $[0,1]$ operation to keep the image valid. We assume this is included in model $f$ and do not write it explicitly in Algorithm \ref{alg:bin_search}.} and decide whether we need to further perform a binary search for this direction. To help better understand the underlying mechanism, Figure \ref{fig:fast_check} provides a two-dimensional sketch for the fast check step in Line \ref{line:fast_check} in Algorithm \ref{alg:bin_search}. Suppose we first change the sign of the current $\db_{\text{best}}$ at dimension $1$, resulting a modified direction $\db_{\text{tmp1}}$. The fast check shows that it is a valid attack and it has the potential to further reduce the decision boundary radius. On the other hand, if we change the sign of  $\db_{\text{best}}$ at dimension $2$, resulting a modified direction $\db_{\text{tmp2}}$. The fast check shows that it is no longer a valid attack and the decision boundary radius of direction $\db_{\text{tmp2}}$ can only be worse than the current $r_{\text{best}}$. Therefore, we skip all unnecessary queries that aim to estimate a worse decision boundary radius. Note that in \citet{cheng2020signopt}, a similar check was also presented for slightly perturbed directions. However, they use it as the sign for gradient estimation while we simply drop all unsatisfied radius based on the check result and obtain better efficiency.
Finally, we explain Line \ref{line:d_end} in Algorithm \ref{alg:bin_search}. The choice of $\min(r_{\text{best}}, \|\db\|_2)$ is because initial $r_{\text{best}}$ is $\infty$, in the case where the fast check passes, we should make sure the binary search range is finite.

\begin{figure}[h]
  \centering
  \includegraphics[width=0.75\linewidth]{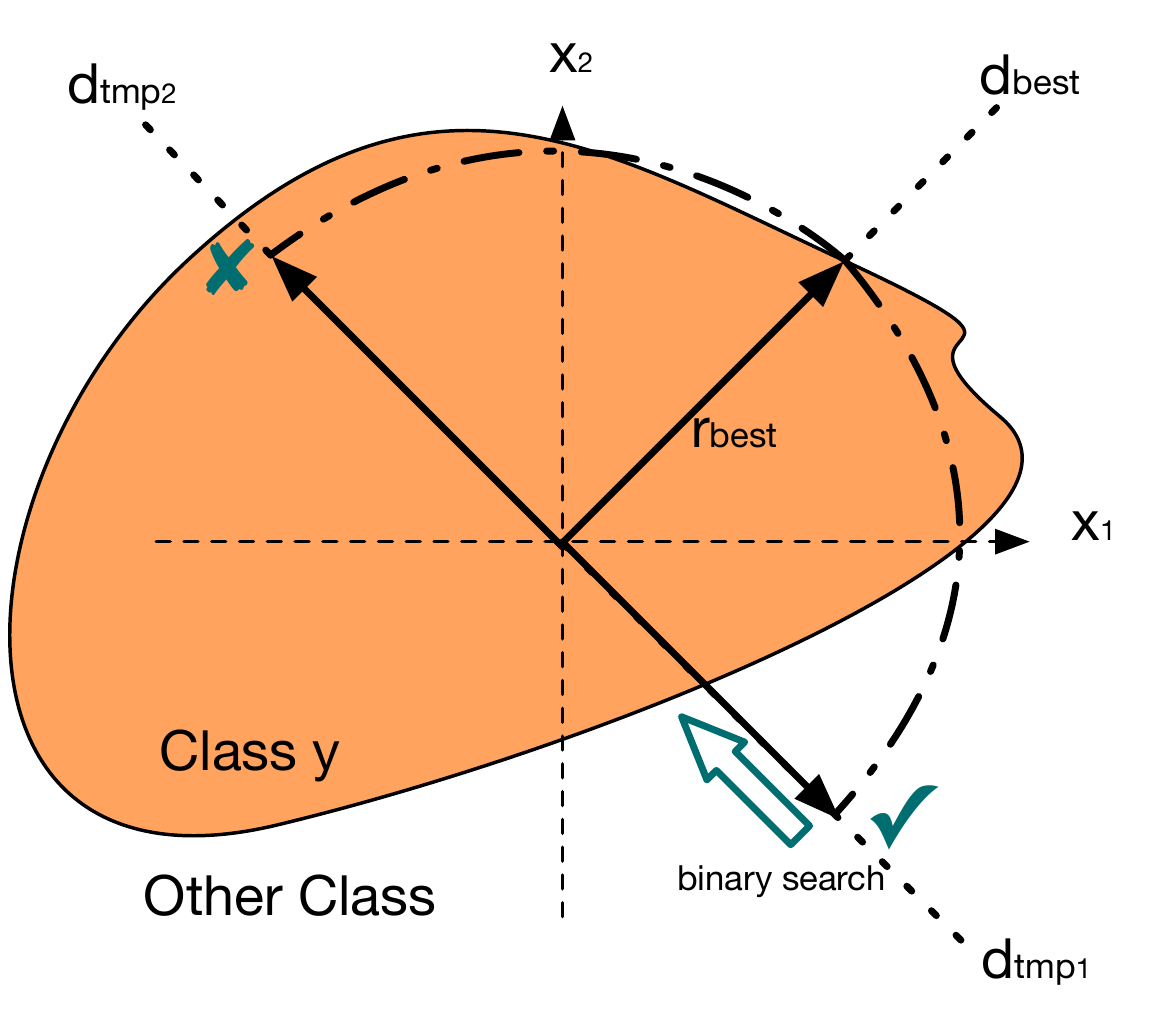}
  \caption{A two-dimensional sketch for the fast check step in Algorithm \ref{alg:bin_search}.}
  \label{fig:fast_check}
\end{figure}

\begin{algorithm}[t]
	\caption{Ray Searching Attack (Naive)}
	\label{alg:rays_naive}
	\begin{algorithmic}[1]
		\STATE \textbf{input:} Model $f$, Original data example $\{\xb, y\}$; 
		\STATE Initialize current best search direction $\db_{\text{best}} = (1, \ldots, 1)$
		\STATE Initialize current best radius $r_{\text{best}} = \infty$
		\STATE Initialize ray searching index $k = 1$
 		\WHILE {remaining query budget $ > 0$}
		\STATE $\db_{\text{tmp}} =   \db_{\text{best}}.copy()$
		\STATE $\db_{\text{tmp}}[k] =   - \db_{\text{tmp}}[k]$
		\STATE $r_{\text{tmp}} = \text{DBR-Search}(f, \xb, y, \db_{\text{tmp}}, r_{\text{best}})  $
		\IF {$r_{\text{tmp}} < r_{\text{best}}$}
        \STATE $r_{\text{best}}, \db_{\text{best}} = r_{\text{tmp}}, \db_{\text{tmp}} $
        \ENDIF
        \STATE $k = k + 1$
        \IF {$k == d$}
        \STATE $k = 1$
        \ENDIF
		\ENDWHILE
		\STATE return $r_{\text{best}}, \db_{\text{best}}$
 	\end{algorithmic}
\end{algorithm}

\begin{algorithm}[t]
	\caption{Decision Boundary Radius Search (DBR-Search)}
	\label{alg:bin_search}
	\begin{algorithmic}[1]
		\STATE \textbf{input:} Model $f$, Original data example $\{\xb, y\}$, Search direction $\db$, Current best radius $r_{\text{best}}$, Binary search tolerance $\epsilon$; 
		\STATE Normalized search direction $\db_n = \db / \|\db\|_2$
		\IF{$f(\xb + r_{\text{best}} \cdot \db_n) == y$}\label{line:fast_check}
		\STATE return $\infty$
		\ENDIF
		
		\STATE Set $start = 0, end = \min(r_{\text{best}}, \|\db\|_2)$ \label{line:d_end}
 		\WHILE { $end - start > \epsilon$} \label{line:bin_start}
 		\STATE $mid = (start + end) / 2$
 		\IF{$f(x + mid \cdot \db_n) == y$}
 		\item {$end = mid$}
 		\ELSE
 		\item {$start = mid$}
 		\ENDIF
		\ENDWHILE\label{line:bin_end}
		\STATE return $end$
 	\end{algorithmic}
\end{algorithm}

\begin{algorithm}[t]
	\caption{Ray Searching Attack (Hierarchical)}
	\label{alg:rays}
	\begin{algorithmic}[1]
		\STATE \textbf{input:} Model $f$, Original data example $\{\xb, y\}$; 
		\STATE Initialize current best search direction $\db_{\text{best}} = (1, \ldots, 1)$
		\STATE Initialize current best radius $r_{\text{best}} = \infty$
		\STATE Initialize stage $s = 0$
		\STATE Initialize block index $k = 1$
		
 		\WHILE {remaining query budget $ > 0$}
		\STATE $\db_{\text{tmp}} =   \db_{\text{best}}.copy()$
		\STATE Cut $\db_{\text{tmp}}$ into $2^s$ blocks and denote index set in the $k$-th block by $\cI_k$
		\STATE $\db_{\text{tmp}}[\cI_k] =   - \db_{\text{tmp}}[\cI_k]$
		\STATE $r_{\text{tmp}} = \text{DBR-Search}(f, \xb, y, \db_{\text{tmp}}, r_{\text{best}})  $
        \IF {$r_{\text{tmp}} < r_{\text{best}}$}
        \STATE $r_{\text{best}}, \db_{\text{best}} = r_{\text{tmp}}, \db_{\text{tmp}} $
        % \STATE $\db_{\text{best}} = \db_{\text{tmp}}$
        \ENDIF
        \STATE $k = k + 1$
        \IF {$k == 2^s$}
        \STATE $s = s + 1$
        \STATE $k = 1$
        \ENDIF
		\ENDWHILE
		\STATE return $r_{\text{best}}, \db_{\text{best}}$
 	\end{algorithmic}
\end{algorithm}

\subsection{Hierarchical Search}
Recent works on black-box attacks \citep{ilyas2018prior, moon2019parsimonious} found that there exists some spatial correlation between different dimensions of the gradients, and exploiting this prior could help improve the efficiency of black-box attacks. Therefore, they added the same perturbation for small tiles or image blocks on the original data example to achieve better efficiency.
Inspired by this finding, we also exploit these spatial correlations by designing a hierarchical search version of the Ray Searching attack, displayed in Algorithm \ref{alg:rays}. Specifically, we add a new stage variable $s$. At each stage, we cut the current search direction into $2^s$ small blocks, and for each iteration, change the sign of the entire block simultaneously as the modified ray search direction for decision boundary radius search. After iterating through all blocks we move to the next stage and repeat the search process. 
Empirically speaking, Algorithm \ref{alg:rays} largely improves the search efficiency by exploiting the spatial correlation mentioned above. All our experiments in Section \ref{sec:exp} are conducted using Algorithm \ref{alg:rays}.
Note that if the query budget is large enough, Algorithm \ref{alg:rays} will, in the end, get to the case where the block size\footnote{For completeness, when $2^s$ is larger than data dimension $d$, Algorithm \ref{alg:rays} will only partition the search direction vector $\db_{\text{tmp}}$ into $d$ blocks to ensure each block contain at least one dimension.} equals to $1$ and reduce to Algorithm \ref{alg:rays_naive} eventually.

Note that all three algorithms (Algorithms \ref{alg:rays_naive}, \ref{alg:bin_search} and \ref{alg:rays}) do not involve any hyperparameters aside from the maximum number of queries, which is usually a predefined problem-related parameter. In sharp contrast, typical white-box attacks and zeroth-order optimization-based black-box attacks, need to tune quite a few hyperparameters in order to achieve good attack performance.
 
\vspace{-0.2cm}
\section{Experiments}\label{sec:exp}
In this section, we present the experimental results of our proposed Ray Searching attack (RayS). We first test RayS attack with other hard-label attack baselines on naturally trained models and then apply RayS attack on recently proposed state-of-the-art robust training models to test their performances. 
All of our experiments are conducted with NVIDIA 2080 Ti GPUs using Pytorch 1.3.1 on Python 3.6.9 platform. 

\vspace{-0.2cm}
\subsection{Datasets and Target Models}
We compare the performance of all attack algorithms on MNIST \citep{lecun2010mnist}, CIFAR-10 \citep{krizhevsky2009learning} and ImageNet \citep{deng2009imagenet} datasets.
Following adversarial examples literature \citep{ilyas2018black, moon2019parsimonious, AlDujaili2020Sign}, we set $\epsilon = 0.3$ for MNIST dataset, $\epsilon = 0.031$ for CIFAR-10 dataset and $\epsilon = 0.05$ for ImageNet dataset.
For naturally trained models, on the MNIST dataset, we attack two pre-trained 7-layer CNN: 4 convolutional layers followed by 3 fully connected layers with Max-pooling and RelU activation applied after each convolutional layer. 
The MNIST pre-trained model achieves $99.5\%$ accuracy on the test set.
On the CIFAR-10 dataset, we also use a 7-layer CNN structure with 4 convolutional layers and an additional 3 fully connected layers accompanied by Batchnorm and Max-pooling layers. The CIFAR-10 pre-trained model achieves $82.5\%$ accuracy on the test set. 
For ImageNet experiments, we attack pre-trained ResNet-50 model \citep{he2016identity} and Inception V3 model  \citep{szegedy2016rethinking}.
The pre-trained ResNet-50 model is reported to have a $76.2\%$ top-$1$ accuracy.
The pre-trained Inception V3 model is reported to have a $78.0\%$ top-$1$ accuracy.  
For robust training models, we evaluate two well-recognized defenses: Adversarial Training (AdvTraining) \citep{madry2017towards} and TRADES \citep{zhang2019theoretically}. In addition, we also test three other recently proposed defenses which claim to achieve the state-of-the-art robust accuracy: Sensible Adversarial Training (SENSE) \citep{kim2020sensible}, Feature Scattering-based Adversarial Training (FeatureScattering) \citep{zhang2019defense},
Adversarial Interpolation Training (AdvInterpTraining) \citep{zhang2020adversarial}.
Specifically, adversarial training \citep{madry2017towards} solves a min-max optimization problem to minimize the adversarial loss.
\citet{zhang2019theoretically} studied the trade-off between robustness and accuracy in adversarial training and proposed an empirically more robust model.
\citet{kim2020sensible} proposed to stop the attack generation when a valid attack has been found. \citet{zhang2019defense} proposed an unsupervised feature-scattering scheme for attack generation.  \citet{zhang2020adversarial} proposed an adversarial interpolation scheme for generating adversarial examples as well as adversarial labels and trained on those example-label pairs.

\vspace{-0.2cm}
\subsection{Baseline Methods}
We compare the proposed algorithm with several state-of-the-art attack algorithms. Specifically, for attacking naturally trained models, we compare the proposed RayS attack with other hard-label attack baselines
(i) OPT attack \citep{cheng2018queryefficient}, 
(ii) SignOPT attack \citep{cheng2020signopt},
and (iii) HSJA attack \citep{chen2019hopskipjumpattack}.
We adopt the same hyperparameter settings in the original papers of OPT, SignOPT, and HSJA attack.

For attacking robust training models, we additionally compare with other state-of-the-art black-box attacks and even white-box attacks:
(i) PGD attack \citep{madry2017towards} (white-box),
(ii) CW attack \citep{carlini2017towards} \footnote{To be precise, here CW attack refers to PGD updates with CW loss \citep{carlini2017towards}} (white-box),
(iii) SignHunter \citep{AlDujaili2020Sign} (black-box),
and (iv) Square attack \citep{andriushchenko2019square} (black-box).
For PGD attack and CW attack, we set step size as $0.007$ and provide attack results for $20$ steps and also $100$ steps. For SignHunter and Square attack, we adopt the same hyperparameter settings used in their original papers.

\subsection{Comparison with hard-label Attack Baselines on Naturally Trained Models}
In this subsection, we compare our Ray Searching attack with other hard-label attack baselines on naturally trained models. 
For each dataset (MNIST, CIFAR-10, and ImageNet), we randomly choose $1000$ images from its test set that are verified to be correctly classified by the pre-trained model and test how many of them can be successfully attacked by the hard-label attacks. For each method, we restrict the maximum number of queries as $10000$. For the sake of query efficiency, we stop the attack for certain test sample once it is successfully attacked, i.e., the $L_\infty$ norm distance between adversarial examples and original examples is less than the pre-defined perturbation limit $\epsilon$.
Tables \ref{tab:untargeted_mnist},   \ref{tab:untargeted_cifar},   \ref{tab:untargeted_resnet} and  \ref{tab:untargeted_inception} present the performance comparison of all hard-label attacks on MNIST model, CIFAR-10 model, ResNet-50 Model and Inception V3 model respectively. For each experiment, we report the average and median of the number of queries needed for successful attacks for each attack, as well as the final attack success rate, i.e., the ratio of successful attacks against the total number of attack attempts.
Specifically, on the MNIST dataset, we observe that our proposed RayS attack enjoys much better query efficiency in terms of average and median of the number of queries, and much higher attack success rate than OPT and SignOPT methods. Note that the average (median) number of queries of SignOPT is larger than that of OPT. However, this does not mean that SignOPT performs worse than OPT. This result is due to the fact that the attack success rate of OPT is very low and its average (median) queries number is calculated based on the successfully attacked examples, which in this case, are the most vulnerable examples. 
HSJA attack, though improving over SignOPT\footnote{Note that the relatively weak performance of SignOPT is due to the fact that SignOPT is designed for $L_2$ norm attack while this experiment is under the $L_\infty$ norm setting. So the result does not conflict with the result reported in the original paper of SignOPT \cite{cheng2020signopt}.}, still falls behind our RayS attack. For the CIFAR model, the RayS attack still achieves the highest attack success rate. Though the HSJA attack comes close to the RayS attack in terms of attack success rate, its query efficiency still falls behind. On ResNet-50 and Inception V3 models, only RayS attack maintains the high attack success rate while the other baselines largely fall behind. Note that HSJA attack achieves similar or even slightly better average (median) queries on ImageNet models, suggesting that HSJA is efficient for the most vulnerable examples but not very effective when dealing with hard-to-attack examples.
Figure \ref{fig:asr} shows the attack success rate against the number of queries plot for all baseline methods on different models. Again we can see that the RayS attack overall achieves the highest attack success rate and best query efficiency compared with other hard-label attack baselines.

\begin{table}[t]
  \caption{Comparison of $L_\infty$ norm based hard-label attack on MNIST dataset ($\epsilon = 0.3$).}
  \label{tab:untargeted_mnist}
  \begin{tabular}{lccc}
    \toprule
    Methods & Avg. Queries & Med. Queries & ASR (\%)  \\
    \midrule
    OPT  & 3260.9 & 2617.0 & 20.9 \\
    SignOPT & 3784.3 & 3187.5 & 62.8  \\
    HSJA   & 161.6 & 154.0 & 91.2  \\
    RayS & \textbf{107.0} & \textbf{47.0} & \textbf{100.0}   \\
    \hline
    PGD (white-box) &  - &  - & 100.0 \\
  \bottomrule
\end{tabular}
\end{table}

\begin{table}[t]
  \caption{Comparison of $L_\infty$ norm based hard-label attack on CIFAR-10 dataset ($\epsilon = 0.031$). }
  \label{tab:untargeted_cifar}
  \begin{tabular}{lccc}
    \toprule
    Methods & Avg. Queries & Med. Queries & ASR (\%)  \\
    \midrule
    OPT & 2253.3 & 1531.0 & 31.0\\
    SignOPT & 2601.3 & 1649.0 & 60.1 \\
    HSJA & 1021.6 & 714.0 & 99.7 \\
    RayS & \textbf{792.8} & \textbf{343.5} & \textbf{99.8} \\
    \hline
    PGD (white-box) &  - & - & 100.0  \\
  \bottomrule
\end{tabular}
\end{table}

\begin{table}[t]
  \caption{Comparison of $L_\infty$ norm based hard-label attack on ImageNet dataset for ResNet-50 model ($\epsilon = 0.05$).}
  \label{tab:untargeted_resnet}
  \begin{tabular}{lccc}
    \toprule
    Methods & Avg. Queries & Med. Queries & ASR (\%)  \\
    \midrule
    OPT & 1344.5 & 655.5 & 14.2\\
    SignOPT & 3103.5 & 2434.0 & 36.0 \\
    HSJA & 749.6 & \textbf{183.0} & 19.9  \\
    RayS & \textbf{574.0} & {296.0} & \textbf{99.8} \\
    \hline
    PGD (white-box) &  - & - & 100.0   \\
  \bottomrule
\end{tabular}
\end{table}

\begin{table}[t]
  \caption{Comparison of $L_\infty$ norm based hard-label attack on ImageNet dataset for Inception V3 model ($\epsilon = 0.05$).}
  \label{tab:untargeted_inception}
  \begin{tabular}{lccc}
    \toprule
    Methods & Avg. Queries & Med. Queries & ASR (\%)  \\
    \midrule
    OPT & 2375.6 & 1674.0 & 21.9  \\
    SignOPT & 2624.8 & 1625.0 & 39.9 \\
    HSJA & \textbf{652.3} & \textbf{362.0} & 23.7 \\
    RayS & 748.2 & 370.0 & \textbf{98.9} \\
    \hline
    PGD (white-box) &  - & - & 100.0 \\
  \bottomrule
\end{tabular}
\end{table}

\begin{figure*}[h]
  \centering
  \subfigure[MNIST]{\includegraphics[width=0.24\linewidth]{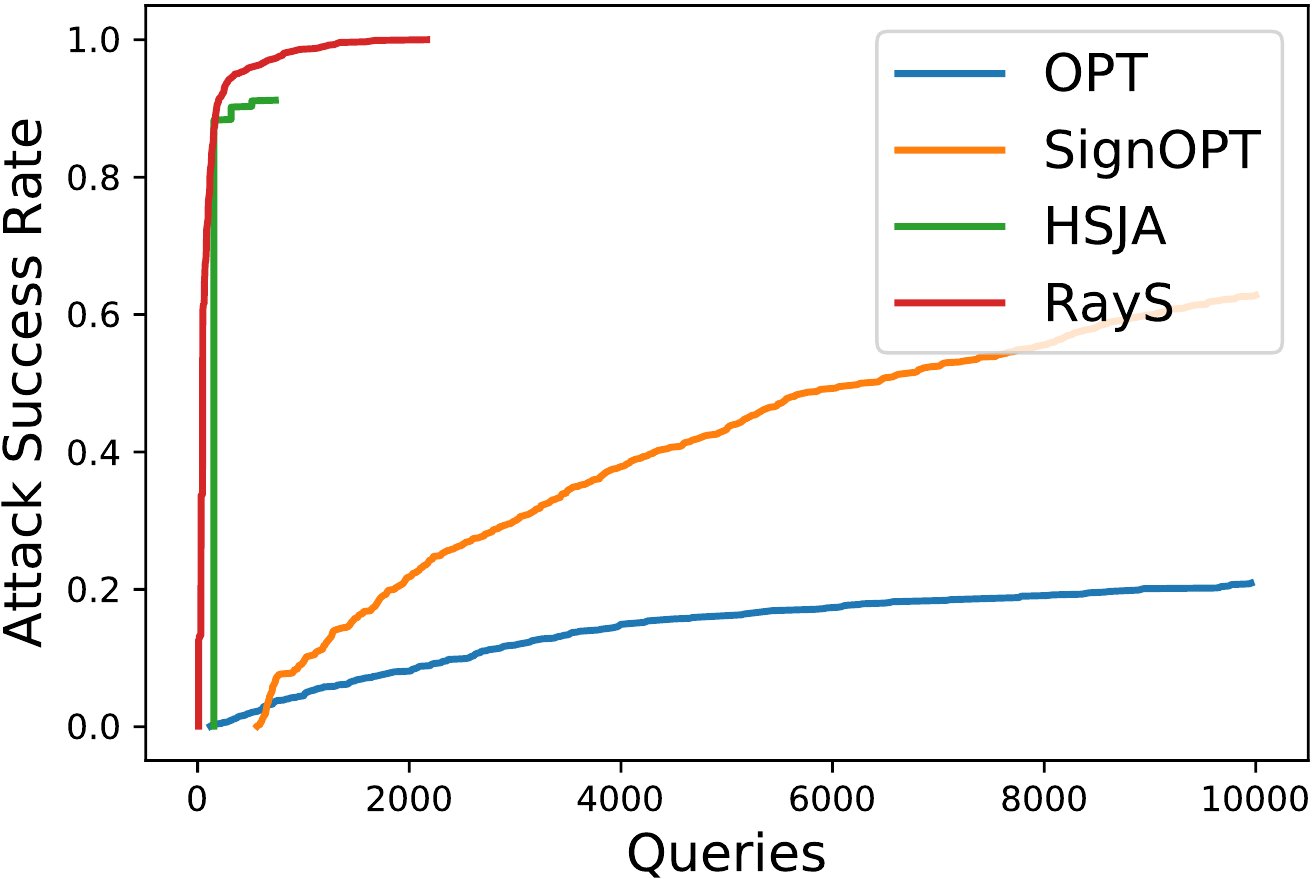}}
  \subfigure[CIFAR]{\includegraphics[width=0.24\linewidth]{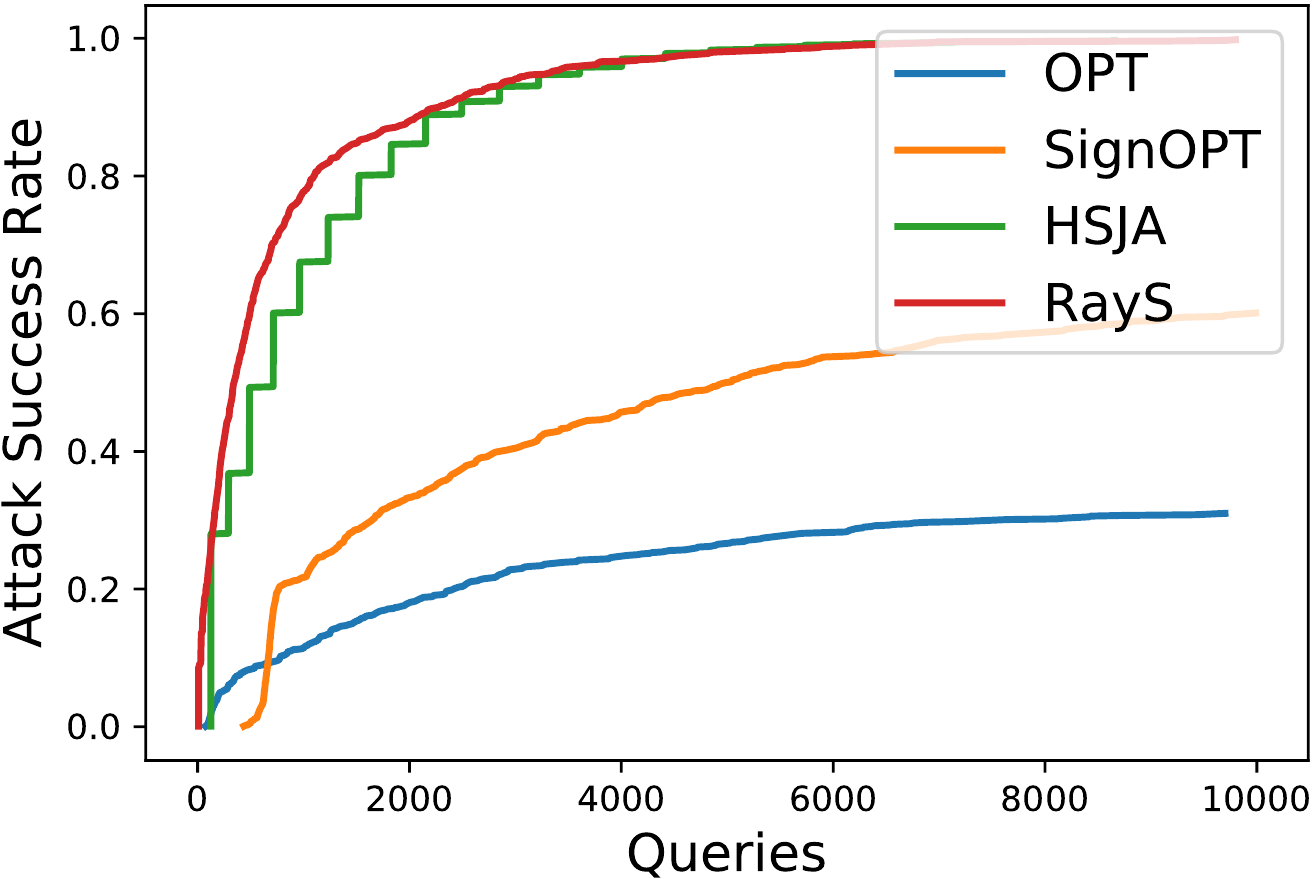}}
  \subfigure[ResNet-50]{\includegraphics[width=0.24\linewidth]{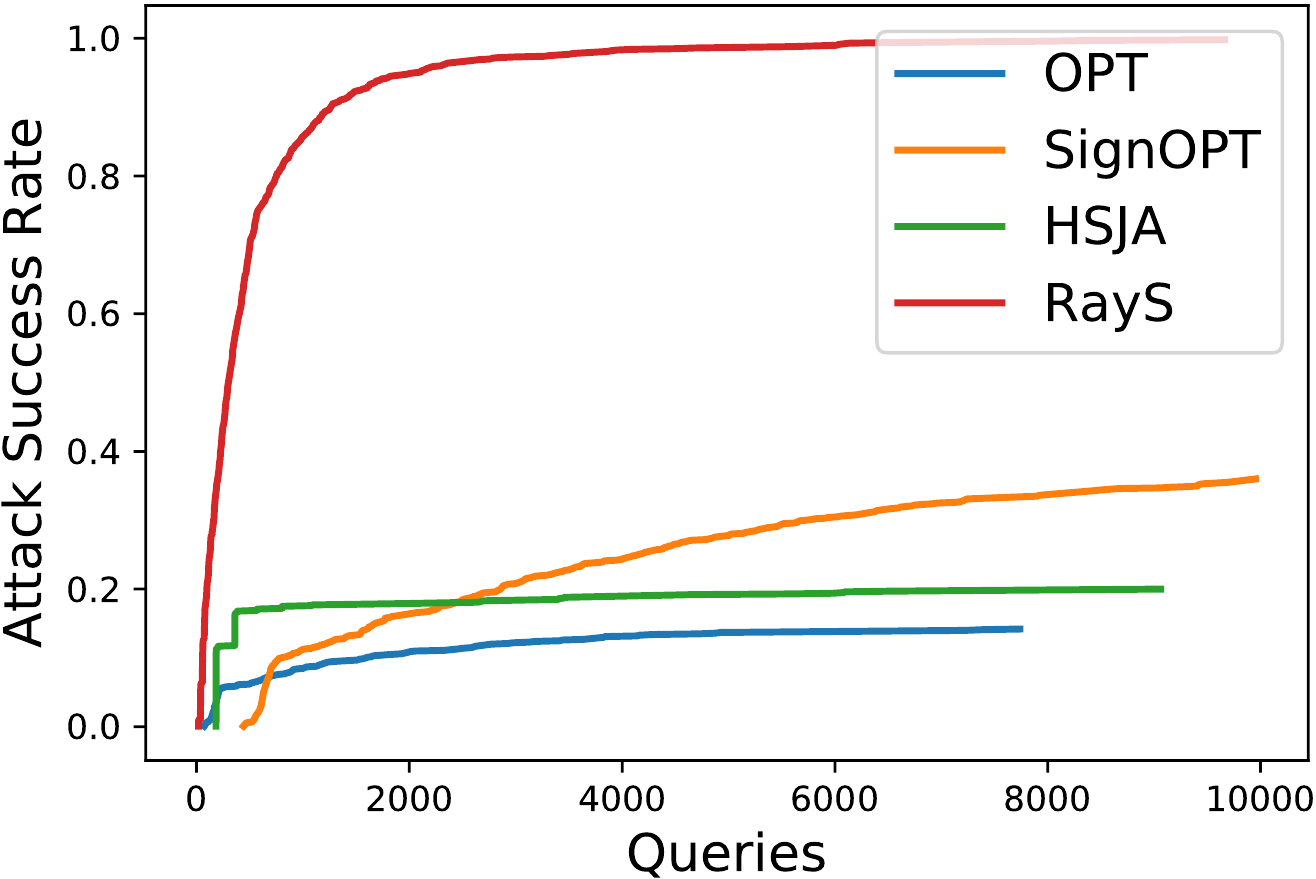}}
  \subfigure[Inception V3]{\includegraphics[width=0.24\linewidth]{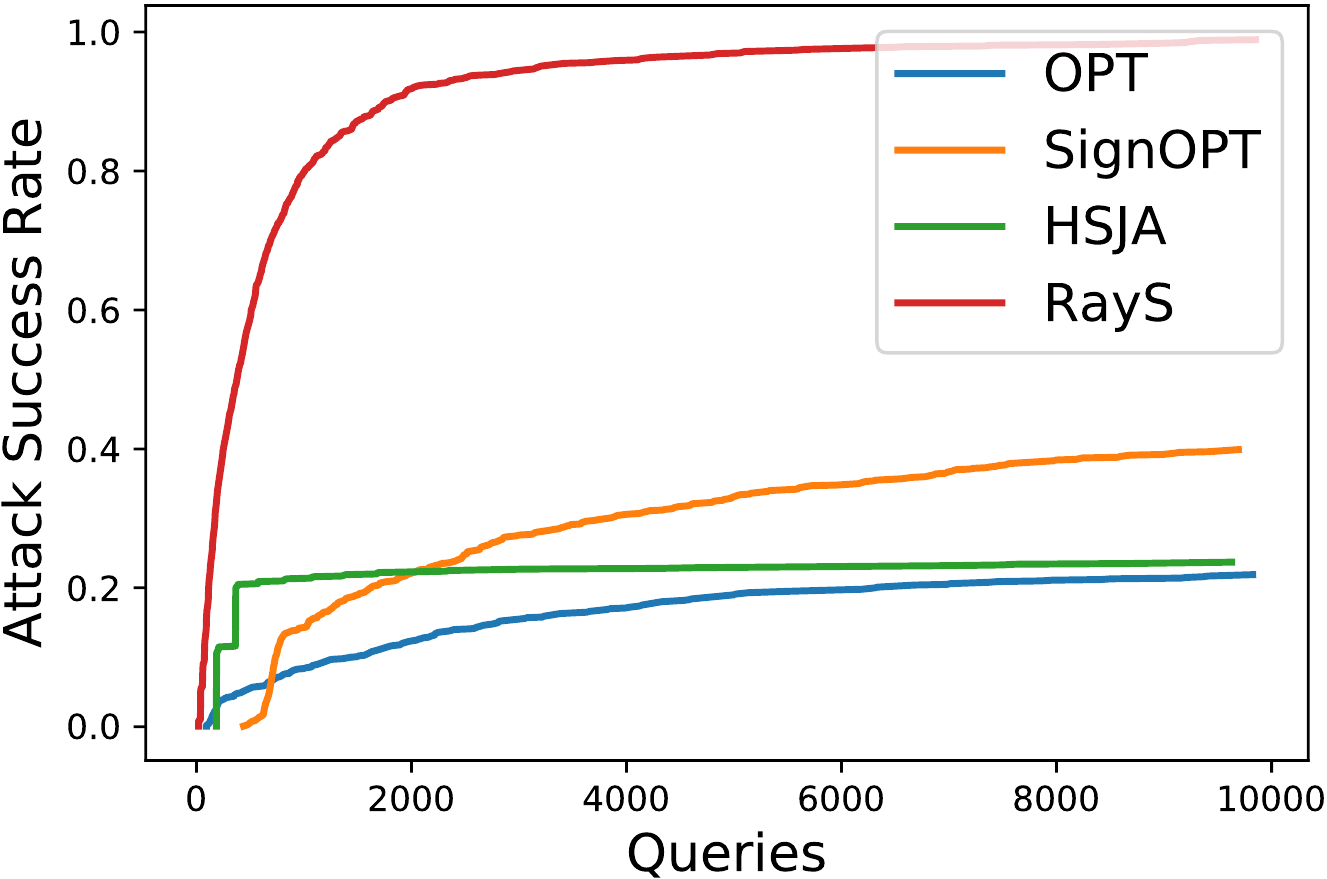}}
  \caption{Attack success rate against the number of queries plots for different hard-label attacks on MNIST, CIFAR-10 and ImageNet datasets. }
  \label{fig:asr}
\end{figure*}

\vspace{-0.2cm}
\subsection{Evaluating the Robustness of State-of-the-art Robust Models}
In this subsection, we further test our proposed Ray Searching attack by applying it to the state-of-the-art robust training models. 
Specifically, we selected five recently proposed open-sourced defenses on CIFAR-10 dataset and WideResNet \citep{zagoruyko2016wide} architecture. For the test examples, we randomly choose $1000$ images from the CIFAR-10 test set. We set the maximum number of queries as $40000$.

In terms of evaluation metrics, following the literature of robust training \citep{madry2017towards,zhang2019theoretically}, we report the natural accuracy and robust accuracy (classification accuracy under adversarial attacks) of the defense model. 
In addition, we report a new metric called \textit{Average Decision Boundary Distance} (ADBD), which is defined as the average $L_\infty$ norm distance between \textit{all} test examples to their nearest decision boundaries. Note that ADBD is not valid for white-box and black-box attacks that follow formulation \eqref{eq:adv_attack}, since they cannot find the nearest decision boundaries for all test examples.

Here we want to emphasize the difference between ADBD and the average $L_\infty$ distortion in the adversarial learning literature. Note that $L_\infty$ distortion\footnote{For all white-box and black-box attacks tested in this experiment, their $L_\infty$ distortions are very close to $0.031$, which is the perturbation limit $\epsilon$. Therefore, we do not report the $L_\infty$ distortion in the tables as it does not provide much additional information.} usually refers to the $L_\infty$ norm distance between successful adversarial attack examples and their corresponding original clean examples and therefore, is affected by the choice the maximum perturbation limit $\epsilon$.
For hard-label attacks, only considering the attacks with a radius less than $\epsilon$ loses too much information and cannot capture the whole picture of model robustness\footnote{ For hard-label attacks, the ADBD value is always larger than the $L_\infty$ distortion.}.
On the other hand, the ADBD metric, though only valid for hard-label attacks, provides a meaningful estimation on the average distance from the original clean examples to their decision boundaries.

Tables \ref{tab:rob_madry}, \ref{tab:rob_trades}, \ref{tab:rob_sense}, \ref{tab:rob_fs} and \ref{tab:rob_advinterp} show the comparison of different adversarial attack methods on five selected robust models. Specifically, for two well recognized robust training models, Adversarial Training (in Table \ref{tab:rob_madry}) and TRADES (in Table \ref{tab:rob_trades}), we observe that white-box attacks are still the strongest attacks, where PGD attack and CW attack achieve very similar attack performances. For black-box attacks, the SignHunter attack and Square attack achieve similar attack performances as their white-box counterparts. In terms of hard-label attacks, our proposed RayS attack also achieves comparable attack performance as black-box or even white-box attacks given the most restricted access to the target model. When comparing with other hard-label attack baselines, it can be seen that our RayS attack achieves significant performance improvement in terms of both robust accuracy (over $20\%$) and the average decision boundary distance (reduced by $30\%$). The less effectiveness in attacking $L_\infty$ norm threat model makes the SignOPT attack and HSJA attack less practical.
For Sensible Adversarial Training model (in Table \ref{tab:rob_sense}), it indeed achieves overall better robust accuracy under white-box attacks, compared with Adversarial Training and TRADES. For black-box attacks, the SignHunter attack achieves similar performance as PGD attack and Square attack achieves similar performance as CW attacks. 
Interestingly, we observe that for hard-label attacks, our proposed RayS attack achieves $42.5\%$ robust accuracy, reducing $20\%$ from PGD attack and $15\%$ from CW attack, suggesting that the robustness of Sensible Adversarial Training is not truly better than TRADES and Adversarial Training, but just looks better under PGD attack and CW attack.
For Feature Scattering-based Adversarial Training model (in Table \ref{tab:rob_fs}), note that the CW attack is much more effective than the PGD attack. Also for black-box attacks, the performance of the Square attack is much better than SignHunter attack\footnote{Square attack is based on CW loss while SignHunter attack is based on CrossEntropy loss.}, suggesting that the CW loss is more effective than CrossEntropy loss in attacking Feature Scattering-based Adversarial Training model. Again, we can observe that our proposed RayS attack reduces the robust accuracy of PGD attack by $28\%$ and CW attack by $10\%$. This also suggests that Feature Scattering-based Adversarial Training model does not really provide better robustness than Adversarial Training or TRADES.
For Adversarial Interpolation Training model (in Table \ref{tab:rob_advinterp}), under white-box attacks, it achieves surprisingly high robust accuracy of $75.3\%$ (under PGD attack) and $68.9\%$ (under CW attack), and similar results can be obtained under the corresponding black-box attacks. However, it is still not truly robust under our RayS attack, reducing the robust accuracy of PGD attack by $28\%$ and CW attack by $22\%$. Note that in this experiment, the HSJA attack also achieves lower robust accuracy than PGD attack, suggesting that all hard-label attacks may have the potential to detect those falsely robust models that deceive current white-box/black-box attacks, but the low efficiency of HSJA restricts its power for greater use.
 
To obtain the overall comparison on the robustness of the five selected robust training models under our proposed RayS attack, we plot the Average Decision Boundary Distance (ADBD) against RayS attack iterations in Figure \ref{fig:adbd} and the robust accuracy against RayS attack iterations in Figure \ref{fig:robacc}. First, it can be seen that the Average Decision Boundary Distance and robust accuracy indeed converge and remain stable after around $10000$ RayS attack iterations. Figures \ref{fig:adbd} and \ref{fig:robacc} suggest that among the five selected robust training models, TRADES and Adversarial Training remain the most robust models while Sensible Adversarial Training, Feature Scattering-based Adversarial Training and Adversarial Interpolation Training, are not as robust as they appear under PGD attacked and CW attack. Note also that even though Sensible Adversarial Training, Feature Scattering-based Adversarial Training and Adversarial Interpolation Training have quite different robust accuracy results under RayS attack, their ADBD results are quite similar.

\setlength{\textfloatsep}{6pt}
\begin{table}[t]
  \caption{Comparison of different adversarial attack methods on Adversarial Training \citep{madry2017towards} for CIFAR-10 dataset (WideResNet, $\epsilon = 0.031$, natural accuracy: $87.4\%$).}
  \label{tab:rob_madry}
  \begin{tabular}{lccc}
    \toprule
    Methods  & Att. Type & ADBD  & Rob. Acc (\%)  \\
    \midrule
    % OPT &  \\
    SignOPT & hard-label & 0.202    & 85.1  \\
    HSJA    & hard-label & 0.060    & 76.8\\
    % RayS & 0.038 &  & 54.0 \\
    RayS & hard-label & \textbf{0.038}    & \textbf{54.0} \\
    \hline
    SignHunter & black-box & -    & \textbf{50.9}\\
    Square & black-box & -    & 52.7\\
    \hline
    PGD-20 & white-box & -    & 51.1  \\
    CW-20 &  white-box & -    & 51.8 \\
    PGD-100 & white-box & -    & \textbf{50.6}   \\
    CW-100  & white-box & -    & 51.5 \\
  \bottomrule
\end{tabular}
\end{table}

% \begin{table}
%   \caption{Untargeted $L_\infty$ norm hard-label attack on CIFAR-10  ($\epsilon = 0.031$) (TRADES-WideResNet)(1000)}
%   \label{tab:rob_trades}
%   \begin{tabular}{lcccc}
%     \toprule
%     Methods & Queries & Dist. & ASR (\%) & Rob. Acc (\%)  \\
%     \midrule
%     % OPT & 1573.26 & 0.025 & 2.30 & 82.72\\
%     % SignOPT & 4509.69 & 0.023 & 4.60 & 80.77\\
%     % HSJA & 4668.70 & 0.026 & 13.50 & 73.24 \\
%     CombSign & 2036.90 & 0.021 & 36.00 & 54.19 \\
%     SignHunter & 1251.09 &  0.031 & 34.10 & 55.80\\
%     PGD (white-box) &  - & 0.031 &  33.10 & 56.64 \\
%     CW (white-box) &  - & 0.031 &  34.70 & 55.29 \\
%   \bottomrule
% \end{tabular}
% \end{table}

\begin{table}[t]
  \caption{Comparison of different adversarial attack methods on TRADES \citep{zhang2019theoretically} for CIFAR-10 dataset (WideResNet, $\epsilon = 0.031$, natural accuracy: $85.4\%$).}
  \label{tab:rob_trades}
  \begin{tabular}{lccc}
    \toprule
    Methods  & Att. Type & ADBD   & Rob. Acc (\%)  \\
    \midrule
    SignOPT & hard-label & 0.196    & 84.0\\
    HSJA  & hard-label & 0.064    & 71.6 \\
    % RayS   & 0.040 &   & 57.3 \\
    RayS   & hard-label & \textbf{0.040}    & \textbf{57.3} \\
    \hline
    SignHunter & black-box &-   & \textbf{56.1}\\
    Square  & black-box &-   & \textbf{56.1}\\
    \hline
    PGD-20  & white-box &-   & 56.5 \\
    CW-20   & white-box &-   & 55.6  \\
    PGD-100 & white-box &-   & 56.3 \\
    CW-100  & white-box &-   & \textbf{55.3}  \\
  \bottomrule
\end{tabular}
\end{table}

\begin{table}[t]
  \caption{Comparison of different adversarial attack methods on SENSE \citep{kim2020sensible} for CIFAR-10 dataset (WideResNet, $\epsilon = 0.031$, natural accuracy: $91.9\%$).}
  \label{tab:rob_sense}
  \begin{tabular}{lccc}
    \toprule
    Methods  & Att. Type & ADBD   & Rob. Acc (\%)  \\
    \midrule
    % OPT  \\
    SignOPT & hard-label & 0.170  & 88.2\\
    HSJA  & hard-label & 0.044   & 66.6\\
    % RayS & 0.029 &  & \textbf{42.5}  \\
    RayS & hard-label & \textbf{0.029}   & \textbf{42.5}  \\
    \hline
    SignHunter  & black-box &-  & 61.9\\
    Square   & black-box &-  & \textbf{58.2}\\
    \hline
    PGD-20  & white-box &-  & 62.1 \\
    CW-20   & white-box &-  & 59.7 \\
    PGD-100 & white-box &-  & 60.1\\
    CW-100  & white-box &-  & \textbf{57.9} \\
  \bottomrule
\end{tabular}
\end{table}

 \begin{table}[t]
  \caption{Comparison of different adversarial attack methods on Feature-Scattering \citep{zhang2019defense} for CIFAR-10 dataset (WideResNet, $\epsilon = 0.031$, natural accuracy: $91.3\%$).}
  \label{tab:rob_fs}
  \begin{tabular}{lccc}
    \toprule
    Methods  & Att. Type & ADBD & Rob. Acc (\%)  \\
    \midrule
    % OPT  \\
    SignOPT & hard-label & 0.175   & 87.1\\
    HSJA  & hard-label & 0.048    & 70.0\\
    % RayS & 0.030 &   & \textbf{44.5}  \\
    RayS & hard-label & \textbf{0.030}    & \textbf{44.5}  \\
    \hline
    SignHunter & black-box &-   & 67.3\\
    Square  & black-box &-   & \textbf{55.3}\\
    \hline
    PGD-20  & white-box &-   & 72.8  \\
    CW-20   & white-box &-   & 57.2  \\
    PGD-100 & white-box &-   & 70.4  \\
    CW-100  & white-box &- & \textbf{54.8}  \\
  \bottomrule
\end{tabular}
\end{table}

\begin{table}[t]
  \caption{Comparison of different adversarial attack methods on Adversarial Interpolation Training \citep{zhang2020adversarial} for CIFAR-10 dataset (WideResNet, $\epsilon = 0.031$, natural accuracy:  $91.0\%$).}
  \label{tab:rob_advinterp}
  \begin{tabular}{lccc}
    \toprule
    Methods  & Att. Type & ADBD & Rob. Acc (\%)  \\
    \midrule
    % OPT  \\
    SignOPT & hard-label & 0.169  & 84.2\\
    HSJA & hard-label & 0.049  & 70.5\\
    % RayS & 0.031 &  & \textbf{46.9} \\
    RayS & hard-label & \textbf{0.031}  & \textbf{46.9} \\
    \hline
    SignHunter & black-box &-  & 73.6\\
    Square  & black-box &-  & \textbf{69.0} \\
    \hline
    PGD-20  &  white-box &-  & 75.6  \\
    CW-20   &  white-box &-  &  69.2  \\
    PGD-100 &  white-box &-  & 75.3 \\
    CW-100  &  white-box &-  &  \textbf{68.9} \\
  \bottomrule
\end{tabular}
\end{table}

\begin{figure}[h]
  \centering
  \includegraphics[width=0.8\linewidth]{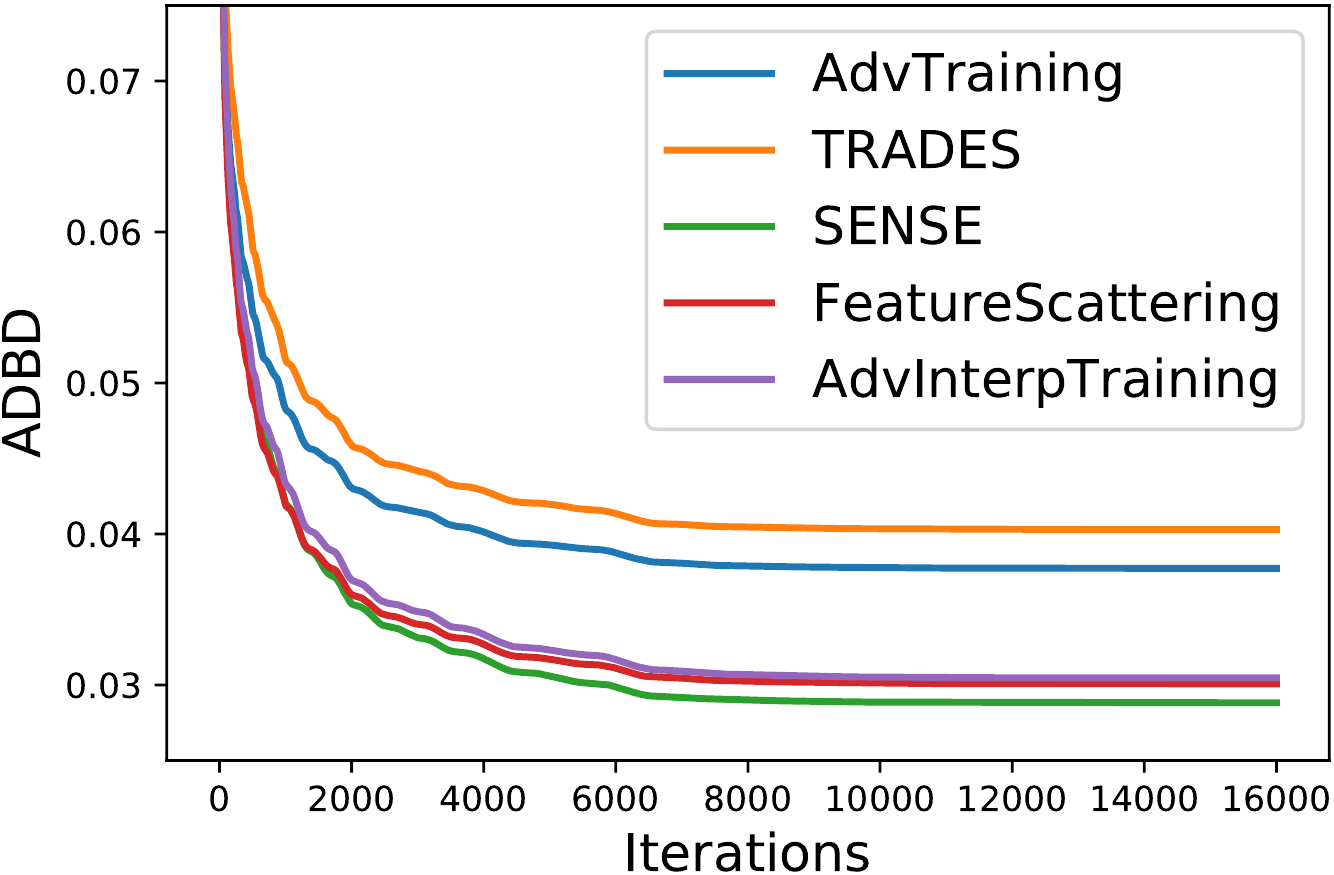}
  \caption{Average Decision Boundary Distance (ADBD) against RayS attack iterations plot for several robust models.}
  \label{fig:adbd}
  \vspace{-0.2cm}
\end{figure}

\begin{figure}[h]
  \centering
  \includegraphics[width=0.8\linewidth]{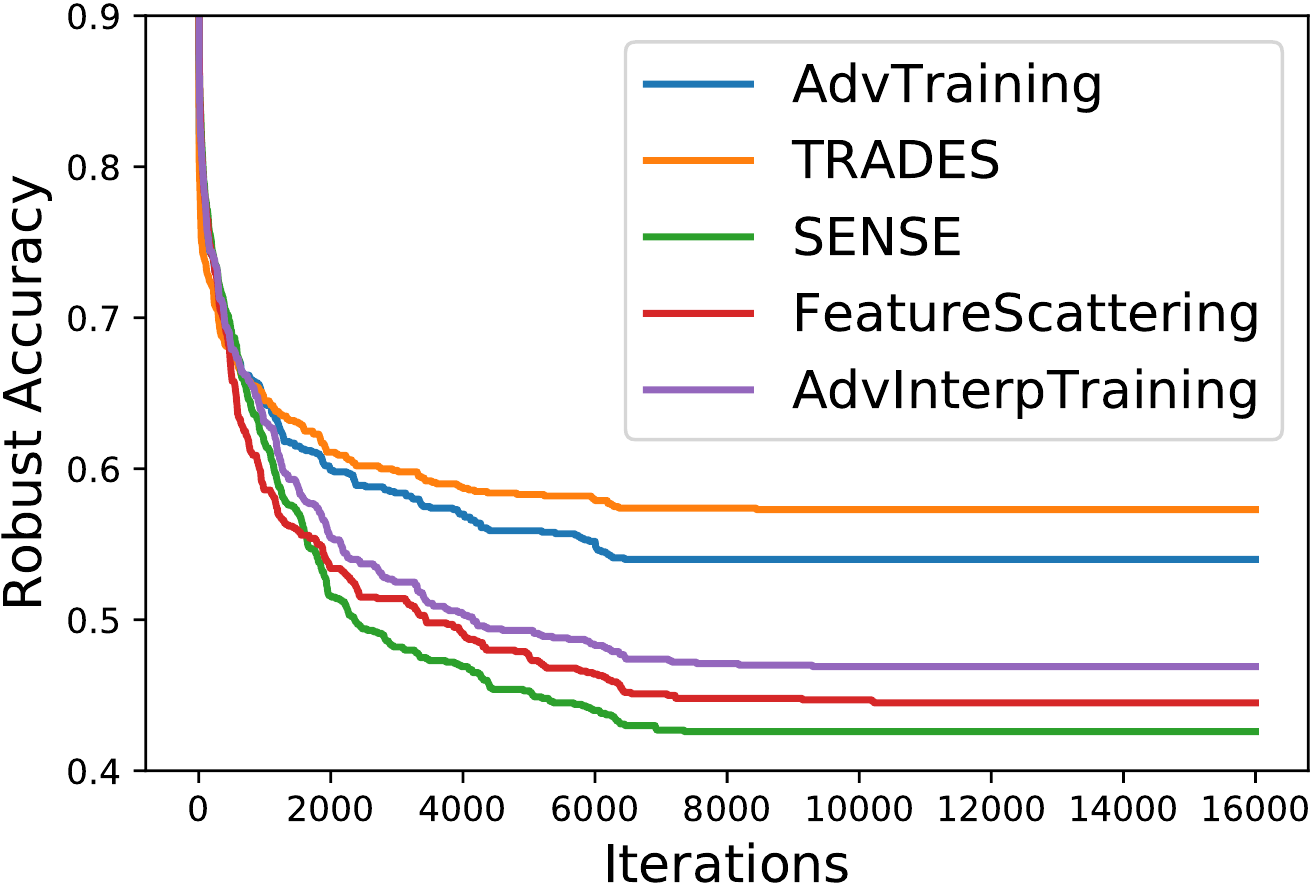}
  \caption{Robust accuracy against RayS attack iterations plot for several robust models.}
  \label{fig:robacc}
  \vspace{-0.2cm}
\end{figure}

% \begin{table}[t]
%   \caption{Comparison of ADBD and robust accuracy  for several robust model under RayS attack.}
%   \label{tab:rays_compare}
%   \begin{tabular}{lcc}
%     \toprule
%     Models  &  ADBD & Rob. Acc (\%)  \\
%     \midrule
%     TRADES   & 0.040   & 57.3 \\
%     AdvTraining & 0.038 & 54.0 \\ 
%     AdvInterpTraining & 0.031 & {46.9} \\
%     FeatureScattering & 0.030  & {44.5}  \\
%     SENSE & 0.029  & 42.5 \\
%   \bottomrule
% \end{tabular}
% \end{table}

\section{Discussions and Conclusions}\label{sec:conclusion}

In this paper, we proposed the Ray Searching attack, which only requires the hard-label output of the target model. The proposed Ray Searching attack is much more effective in attack success rate and efficient in terms of query complexity, compared with other hard-label attacks. Moreover, it can be used as a sanity check tool for possible ``falsely robust'' models that deceive current white-box and black-box attacks.

In the following discussions, we try to analyze the key ingredients for the success of the proposed Ray Searching attack.

\textit{Why RayS attack is more effective and efficient than the other hard-label baselines?}

As we mentioned before, traditional hard-label attacks are more focused on the $L_2$ norm threat model with only a few extensions to the $L_\infty$ norm threat model. 
% As a result, their algorithms are not optimized and did not consider the unique properties in the $L_\infty$ norm threat model. 
While for our RayS attack, we reformulate the continuous problem of finding the closest decision boundary into a discrete problem based on empirical findings in $L_\infty$ norm threat model, which leads to a more effective hard-label attack. On the other hand, the strategy of directly searching for the closest decision boundary together with a fast check step eliminates unnecessary searches and significantly improves the attack efficiency. 
 
\textit{Why RayS attack can detect possible ``false'' robust models while traditional white-box and black-box attacks cannot?}

One thing we observe from Section \ref{sec:exp} is that although different attacks lead to different robust accuracy results, their attack performances are correlated with the choice of attack loss functions, e.g., both PGD attack and SignHunter attack utilize CrossEntropy loss and their attack performances are similar in most cases. A similar effect can also be seen for the CW attack and Square attack, both of which utilize the CW loss function. 
However, these loss functions were used as surrogate losses to problem \eqref{eq:adv_attack}, and they may not be able to truly reflect the quality/potential of an intermediate example (an example near the original clean example that is not yet a valid adversarial example).
For instance, consider the case where two intermediate examples share the same log probability at ground truth class $y$, but vary drastically on other classes.
Their CrossEntropy losses are the same in such cases, but one may have larger potential to develop into a valid adversarial example than the other one (e.g., the second-largest probability is close to the largest probability). Therefore, CrossEntropy loss does not really reflect the true quality/potential of the intermediate examples. Similar instances can also be constructed for CW loss. In sharp contrast, our RayS attack consider the decision boundary radius as the search criterion\footnote{Actually it is a criterion for all hard-label attack.}. When we compare two examples on the decision boundary, it is clear that the closer one is better. 
In cases where the attack problem is hard to solve and the attacker could easily get stuck at intermediate examples (e.g., attacking robust training models), it is easy to see that the RayS attack stands a better chance of finding a successful attack.
This partially explains the superiority of RayS attack in detecting ``falsely robust'' models.

\begin{acks}
We thank the anonymous reviewers and senior PC for their helpful comments. This research was sponsored in part by the National Science Foundation CIF-1911168, SaTC-1717950, and CAREER Award 1906169. 
The views and conclusions contained in this paper are those of the authors and should not be interpreted as representing any funding agencies. 
\end{acks}

% \newpage

\bibliographystyle{ACM-Reference-Format}
\bibliography{adv}

\appendix
 
\end{document}